\documentclass[accepted]{uai2023-template/uai2023}
\usepackage[american]{babel}

\usepackage[usenames,dvipsnames]{xcolor}

\usepackage{natbib}
\usepackage{mathtools} %
\usepackage{booktabs}
\usepackage{tikz}
\usepackage{multirow}

\usepackage[colorlinks,linktoc=all]{hyperref}       
\hypersetup{citecolor=MidnightBlue}
\hypersetup{urlcolor=MidnightBlue}
\hypersetup{linkcolor=black}

\usepackage{amsfonts}
\usepackage{amsthm}
\usepackage[ruled, linesnumbered]{algorithm2e}
\usepackage{setspace}
\usepackage{colortbl}
\definecolor{Gray}{gray}{0.9}

\usepackage{floatrow}

\newtheorem{theorem}{Theorem}[section]
\newtheorem{proposition}[theorem]{Proposition}

\newtheorem{remark}[theorem]{Remark}
\newtheorem{lemma}[theorem]{Lemma}

\title{The Shrinkage-Delinkage Trade-off: An Analysis of \\ Factorized Gaussian Approximations for Variational Inference}

\author[1]{\href{mailto:<cmargossian@flatironinstitute.org>?Subject=Your UAI 2023 paper}{Charles C. Margossian}{}}
\author[1]{Lawrence K. Saul}

\affil[1]{%
    Center for Computational Mathematics\\
    Flatiron Institute\\
    New York, NY, USA
}

\begin{document}
  
\maketitle


\begin{abstract}
When factorized approximations are used for variational inference (VI), they tend to underestimate the uncertainty---as measured in various ways---of the distributions they are meant to approximate.
We consider two popular ways to measure the uncertainty deficit of VI: (i) the degree to which it underestimates the componentwise variance, and (ii) the degree to which it underestimates the entropy.
To better understand these effects, and the relationship between them, we examine an informative setting where they can be explicitly (and elegantly) analyzed:
the approximation of a Gaussian,~$p$, with a dense covariance matrix, by a Gaussian,~$q$, with a diagonal covariance matrix.
We prove that $q$ always underestimates both the componentwise variance and the entropy of $p$, \textit{though not necessarily to the same degree}.
Moreover we demonstrate that the entropy of $q$ is determined by the trade-off of two competing forces: it is decreased by the shrinkage of its componentwise variances (our first measure of uncertainty) but it is increased by the factorized approximation which delinks the nodes in the graphical model of $p$.
We study various manifestations of this trade-off, notably one where, as the dimension of the problem grows, the per-component entropy gap between $p$ and $q$ becomes vanishingly small even though $q$ underestimates every componentwise variance by a constant multiplicative factor.
We also use the shrinkage-delinkage trade-off to bound the entropy gap in terms of the problem dimension and the condition 
number of the correlation matrix of $p$.
Finally we present empirical results on both Gaussian and non-Gaussian targets, the former to validate our analysis and the latter to explore its limitations.
\end{abstract}


\section{Introduction}

Variational inference (VI) is a popular methodology for approximate Bayesian inference \citep{Jordan:1999, Wainwright:2008, Blei:2017}.
Given a target distribution,~$p$, VI searches for a tractable distribution, $q \in \mathcal Q$, that minimizes the Kullback-Leibler (KL) divergence to~$p$. 
A common choice for $\mathcal Q$ is to use a family of factorized distributions.
The KL-divergence can then be optimized in a scalable manner for high-dimensional distributions \citep{Wainwright:2008}, which is crucial, for instance, to train models such as variational auto-encoders over large data sets \citep{Kingma:2013}.

Factorized VI has its roots in the mean-field approximations to certain Gibbs distributions from statistical physics~\citep{Parisi:1988,Mackay:2003}. 
In this approach, the approximating distribution is modeled as
\begin{equation}
  q({\bf z}) = \prod_{i = 1}^n q(z_i).
\label{eq:FVI}
\end{equation}
In most applications, the target distribution $p({\bf z})$ does \underline{not} factorize.
By its very nature, factorized VI cannot estimate the correlations between different elements of ${\bf z}$.
A more subtle shortcoming of factorized VI is that it also fails to correctly estimate the marginal distributions, $p(z_i)$.
This failure typically manifests as an approximation $q$ with an uncertainty deficit relative to $p$, a phenomenon which has been studied both empirically and theoretically~\citep[e.g][]{Mackay:2003, Wang:2005, Bishop:2006, Turner:2011, Blei:2017, Giordano:2018}.
There exists several measures of uncertainty, and we focus on two: (i) the componentwise variance and (ii) the entropy.
The componentwise variance plays a crucial role in Bayesian modeling, especially when estimating the posterior distribution over interpretable variables.
Meanwhile the entropy provides a multivariate notion of uncertainty and, in statistical physics, can be linked to the free energy, a quantity of interest for many problems.

Intuitively, we expect factorized VI to shrink the variance of~$q$ to minimize its overlap with the tails of $p$.
It is less clear how it should affect the entropy of $q$: on the one hand, this entropy is decreased by any shrinkage in the variance, but it is increased
by the factorized approximation, which delinks the nodes in the full-covariance graphical model of $p$.
Hence entropy is driven by a trade-off between two competing forces. We call this the \textit{shrinkage-delinkage trade-off}.
This trade-off hints that the adequacy of factorized VI may depend on the way we elect to measure its uncertainty deficit.

The goal of this paper is to understand the uncertainty deficit of factorized VI in the most informative setting where it can be rigorously analyzed.
To this end, we study the special case where $p$ is a Gaussian distribution over $\mathbb R^n$ with a full covariance matrix and $q$ is a Gaussian distribution over~$\mathbb R^n$ with a diagonal covariance matrix.
This choice of $q$ is natural when $p$ is a multivariate distribution over ${\bf z} \in \mathbb R^n$, and leads to factorized Gaussian variational inference (FG-VI)---a popular method among practionners due notably to ``black box'' implementations such as automatic differentiation variational inference (ADVI) \citep{Kucukelbir:2017}.
Our paper expands on previous analyses of FG-VI \citep{Bishop:2006, Turner:2011} in many ways, but perhaps most significantly by identifying---and elucidating---the shrinkage-delinkage trade-off of factorized VI, which in the considered setting can be written explicitly.

Our analysis is grounded in two fundamental inequalities. First we show that if $p$ is multivariate Gaussian, and if~$q$ is the distribution (optimally) estimated by FG-VI, then
\begin{equation}
\text{Var}_q(z_i) \le \text{Var}_p(z_i). 
\label{eq:ineq1}
\end{equation}
Second, under the same assumptions, we show~that
\begin{equation}
{\cal H}(q) \le {\cal H}(p),
\label{eq:ineq2}
\end{equation}
where ${\cal H}(\cdot)$ denotes the entropy. 
This second inequality, relating the entropies of $p$ and $q$, formalizes an observation~\citep{Mackay:2003,Bishop:2006} that $q$ tends to be more ``compact'' than~$p$.
Our proofs of these inequalities hold generally for Gaussian distributions over $\mathbb{R}^n$; to the best of our knowledge, they are more direct and more general than previous demonstrations.
While both inequalities reveal an uncertainty deficit, we will see that the two notions of uncertainty are not equivalent.
Indeed, we provide one example where $q$ underestimates each componentwise variance by a constant multiplicative factor, but the per-component entropy gap between $p$ and $q$ can be arbitrarily small.
This discrepancy arises because the entropy gap in FG-VI is in fact \textit{equal} to the KL divergence minimized by FG-VI when it targets a multivariate Gaussian.
But, as we will see, this choice of objective function can harm the estimation of marginal variances.


The inequalities in eq.~(\ref{eq:ineq1}--\ref{eq:ineq2}) anchor our subsequent analysis. As shown in Figure~\ref{fig:shrinkage}, the amount of shrinkage in FG-VI depends in general on the number of components of~${\bf z}\in\mathbb{R}^n$ as well as the degree of correlation between these components. 
With this motivation, we derive an upper bound on the entropy gap in eq.~(\ref{eq:ineq2}) in terms of the problem dimensionality,~$n$, and the condition number of the true correlation matrix.

Finally we examine the relevance for some of our findings when FG-VI is applied to non-Gaussian target distributions. For these experiments, we draw on several examples from the Bayesian literature.
We find that, while the variance shrinkage (\ref{eq:ineq1}) does not hold systematically, it holds on average in the considered examples.
We do not have a reliable method to empirically estimate the entropy, but make an argument that eq.~(\ref{eq:ineq2}) may hold in the studied examples.

Our results build on those of many previous studies. \citet{Mackay:2003}, \citet{Bishop:2006}, \citet{Turner:2011}, and \citet{Blei:2017} all use a two-dimensional Gaussian to illustrate that the approximations from VI are more ``compact'' than the distributions they target. We formalize this observation in the general $n$-dimensional setting, while highlighting the difference between componentwise variance and entropy as measures of uncertainty---a difference that becomes more critical in high-dimensional settings.
Experiments on non-Gaussian models also suggest a more nuanced picture, showing for instance that FG-VI does not always underestimate every componentwise variance, though in the studied examples variance shrinkage holds \textit{on average}. 
Previous studies have also examined other measure of uncertainty, such as the frequentist intervals obtained by variational Bayes estimators \citep[e.g.][]{Wang:2005}. Finally, many have been motivated by the uncertainty deficit of factorized VI to develop new methods for inference. These include post-hoc corrections of variational approximations \citep{Giordano:2018} or, for certain models, careful decompositions of $p$ using conditional distributions to justify the assumption of factorization \citep{Agrawal:2021}.

The code for all results and figures is available on \href{https://github.com/charlesm93/variance-delinkage}{GitHub}.

\begin{figure*}
    \centering
    \includegraphics[width = 6in]{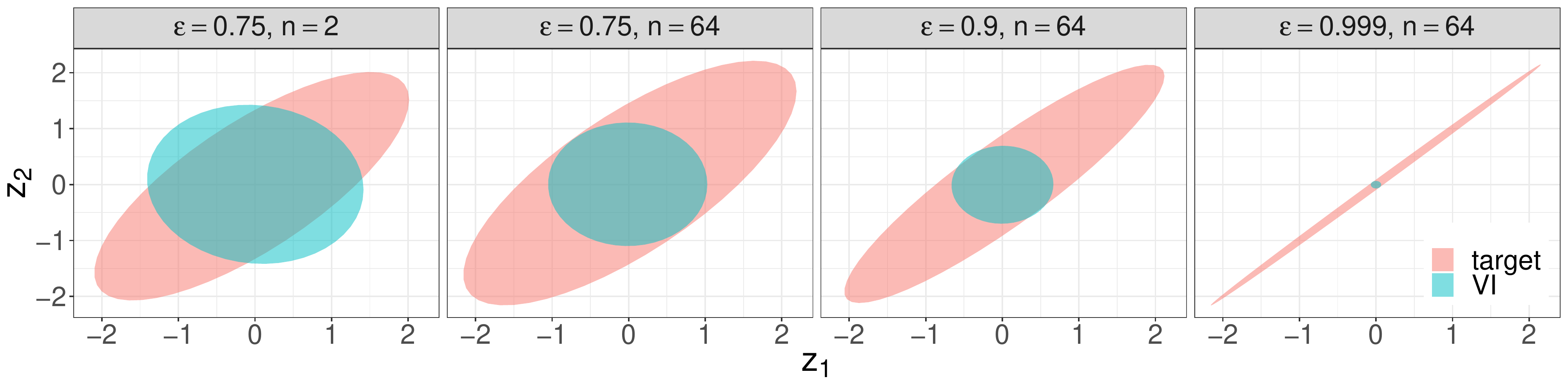}
    \caption{\textit{FG-VI's approximation of a multivariate Gaussian whose correlation matrix has constant off-diagonal terms. FG-VI's variance shrinkage grows with both increasing dimensionality ($n$) and correlation ($\varepsilon$). For $n=64$, the distributions are projected onto their first two coordinates.
    Despite what the picture suggests, the entropy gap between the approximation and the target is actually quite small (Section~\ref{sec:corr}).
    In this sense, the lower-dimensional projection is misleading.
    }}
    \label{fig:shrinkage}
\end{figure*}


\section{Preliminaries}

We analyze FG-VI in the setting where $p(\mathbf{z})$ is multivariate Gaussian with mean $\boldsymbol\mu\in\mathbb{R}^n$ and covariance $\boldsymbol\Sigma\in\mathbb{R}^{n\times n}$. In this setting FG-VI has a particularly simple solution. (An earlier statement of this solution can be found in \citet{Turner:2011}.)

\begin{proposition} \label{prop:solution}
 Let $q(\mathbf{z})$ be multivariate Gaussian with mean $\boldsymbol\nu$ and diagonal covariance~$\boldsymbol\Psi$.
 Then the variational parameters minimizing $\text{KL}(q||p)$ are given by $\boldsymbol\nu=\boldsymbol\mu$ and
\begin{equation}
    \Psi_{ii} = \frac{1}{\Sigma^{-1}_{ii}},
\label{eq:Psi}
\end{equation}
where the denominator $\Sigma_{ii}^{-1}$ denotes a diagonal element of the matrix inverse~$\boldsymbol\Sigma^{-1}$.
\end{proposition}
\begin{proof}
The variational parameters $\boldsymbol\nu$ and $\boldsymbol\Psi$ are estimated by minimizing the KL-divergence
\begin{equation}
  \text{KL}(q || p) = \mathbb E_q[\log q({\bf z})] - \mathbb E_q[\log p({\bf z})],
  \label{eq:KL}
\end{equation}
where each expectation is taken with respect to the measure~$q$. Note that only the second term in eq.~(\ref{eq:KL}) depends on the variational mean $\boldsymbol\nu$, and it is given by
\begin{equation}
-\mathbb E_q[\log p({\bf z})] = \tfrac{1}{2}(\boldsymbol\nu\!-\!\boldsymbol\mu)^\top\boldsymbol\Sigma^{-1}(\boldsymbol\nu\!-\boldsymbol\mu)\, +\, \ldots
\end{equation}
where the ellipses indicate terms that do not depend on~$\boldsymbol\nu$. By minimizing this expression, it follows at once that $\boldsymbol\nu=\boldsymbol\mu$. With this substitution, eq.~(\ref{eq:KL}) simplifies to
\begin{equation}
 \text{KL}(q || p) = \tfrac{1}{2}\left[\text{trace}\big(\boldsymbol\Psi\boldsymbol\Sigma^{-1}\big) - \log\big|\boldsymbol\Psi\boldsymbol\Sigma^{-1}\big| - n\right],
\label{eq:KL-gaussian}
\end{equation}
and the result in eq.~(\ref{eq:Psi}) follows by minimizing the above expression with respect to the diagonal elements of $\boldsymbol\Psi$.
\end{proof}

In sections 2, 3, and 4 of the paper, we assume that $q$ is the factorized Gaussian distribution whose variances are given by eq.~(\ref{eq:Psi}). {\it We emphasize in general that} $\Psi_{ii}\neq\Sigma_{ii}$. However, it is true that $\boldsymbol\Psi=\boldsymbol\Sigma$ when $\boldsymbol\Sigma$ is diagonal.

Many of our results will not be expressed directly in terms of $\boldsymbol\Sigma$ and $\boldsymbol\Psi$, but in terms of two related (but dimensionless) matrices. The first is the {\it correlation matrix}~$\mathbf{C}$ with elements
\begin{equation}
C_{ij} = \frac{\Sigma_{ij}}{\sqrt{\Sigma_{ii}\Sigma_{jj}}}.
\label{eq:C}
\end{equation}
Note that $C_{ii}=1$, a simple fact that we will often exploit, and also that $\mathbf{C}$ reduces to the identity matrix when $\boldsymbol\Sigma$ is diagonal.
At the other extreme, we may consider the case where all the off-diagonal elements of $\mathbf{C}$ are equal to some constant $\varepsilon\!>\!0$. This is explored visually in Figure.~(\ref{fig:shrinkage}). In appendix A we show that $\Psi_{ii}\rightarrow 0$ as $\varepsilon\rightarrow 1$ for fixed $n$, and that $\Psi_{ii}\rightarrow (1\!-\!\varepsilon)\Sigma_{ii}$ as $n\rightarrow\infty$ for fixed~$\varepsilon$. Note that FG-VI underestimates the variance in both limits.

In addition to the correlation matrix, we also define the diagonal \textit{shrinkage} matrix~$\mathbf{S}$ with dimensionless entries
\begin{equation}
    S_{ii} = \frac{\Sigma_{ii}}{\Psi_{ii}} = \Sigma_{ii}\,\Sigma^{-1}_{ii}.
    \label{eq:S}
\end{equation}

We will use the matrices $\mathbf{C}$ and $\mathbf{S}$ to analyze how FG-VI underestimates the uncertainty of $p$. The uncertainty in axis-aligned directions is measured by the variances $\Sigma_{ii}$, but a multivariate measure of uncertainty is provided by the entropy
\begin{equation}
  \mathcal H(p) = - \mathbb E_p \log p({\bf z}).
\label{eq:entropy}
\end{equation}
To what extent does FG-VI underestimate this entropy? As shown next, the answer is very naturally expressed in terms of the correlation matrix $\mathbf{C}$ and the shrinkage matrix $\mathbf{S}$.

\begin{proposition} 
Let $p$ and $q$ be defined as above. Then their difference in entropy is given by
    \begin{equation}
      \mathcal H(p) - \mathcal H(q)
      = \tfrac{1}{2} \log |\mathbf{S}| - \tfrac{1}{2} \log |\mathbf{C}|^{-1}.
      \label{eq:entropy-loss}
    \end{equation}
\end{proposition}
\begin{proof}
A standard calculation for multivariate Gaussian distributions~\citep{Cover:2006} gives 
\mbox{$\mathcal H(p) = \frac{1}{2} \log |\Sigma|(2\pi e)^n$}, and
an analogous result holds for~$\mathcal H(q)$. Let $\Delta\mathcal H = \mathcal H(p)-\mathcal H(q)$. Then we see that
\begin{equation}
  \Delta\mathcal H 
    = \tfrac{1}{2} \log|\mathbf{\Sigma}| - \tfrac{1}{2}|\log{\mathbf{\Psi}}| 
    = \tfrac{1}{2}\log \left|\mathbf{\Psi}^{-\frac{1}{2}} \mathbf{\Sigma} \mathbf{\Psi}^{-\frac{1}{2}}\right|.
    \label{eq:entropy-loss2}
  \end{equation}
Now from the definitions in eqs.~(\ref{eq:C}--\ref{eq:S}), it can be verified by direct substitution that
    $\boldsymbol\Psi^{-\frac{1}{2}} \boldsymbol\Sigma\, \boldsymbol\Psi^{-\frac{1}{2}} = 
    \mathbf{S}^\frac{1}{2}\mathbf{C}\,\mathbf{S}^\frac{1}{2}$.
It follows from the basic properties of determinants that
\begin{equation*}
  \Delta\mathcal H 
  = \tfrac{1}{2} \log \big|\mathbf{S}^\frac{1}{2} \mathbf{C} \mathbf{S}^\frac{1}{2}\big| = 
  \tfrac{1}{2}\log|\mathbf{S}| - \tfrac{1}{2}\log|\mathbf{C}|^{-1}.
\end{equation*}
\end{proof}

\section{\mbox{\!\!Shrinkage-Delinkage Trade-off}} \label{sec:trade-off}

In this section we prove that FG-VI systematically underestimates the variance and entropy of a multivariate Gaussian distribution. We will see, however, that a large shrinkage in {\it all} componentwise variances does not imply a correspondingly large shrinkage in the entropy.

\subsection{Fundamental inequalities}

  %
\begin{theorem}[Variance shrinkage]
\label{thm:shrinkage-variance}
The solution for FG-VI in eq.~(\ref{eq:Psi}) underestimates the variance; that is,
\begin{equation}
    \Psi_{ii} \le \Sigma_{ii},
\label{eq:variance-shrinkage}
\end{equation}
and the inequality is strict for some component of the variance (i.e., $\Psi_{ii}<\Sigma_{ii}$) if $\mathbf{\Sigma}$ is not purely diagonal.
\end{theorem}
\begin{proof}
Let $\mathbf{C}$ denote the correlation matrix in eq.~(\ref{eq:C}).
It can be verified by direct calculation that
\begin{equation}
 C^{-1}_{ij} = \Sigma^{-1}_{ij}\sqrt{\Sigma_{ii}\Sigma_{jj}}.
 \label{eq:invC}
\end{equation}
As further notation, let $\lambda_1,\ldots,\lambda_n$ denote the eigenvalues of $\mathbf{C}$, and
let $\mathbf{e}_i$ denote the unit vector along the $i^{\rm th}$ axis. Then from the solution in eq.~(\ref{eq:Psi}), it follows that
\begin{eqnarray}
  \frac{\Sigma_{ii}}{\Psi_{ii}}
    &=& C^{-1}_{ii}, \label{eq:Sii} \\
    &=& C^{-1}_{ii} + C_{ii} - 1 \\
    &=& \mathbf{e}_i^\top(\mathbf{C}^{-1}\! + \mathbf{C}\,)\,\mathbf{e}_i - 1, \\
    &\geq& \min_{\|\mathbf{e}\|=1} [\mathbf{e}^\top(\mathbf{C}^{-1}\! + \mathbf{C}\,)\,\mathbf{e} - 1], 
    \label{eq:eigen-ineq} \\
    &=& \min_i\, (\lambda_i^{-1} + \lambda_i - 1), \\
    &\geq& \min_{\lambda>0}\, (\lambda^{-1} + \lambda - 1) = 1, 
\end{eqnarray}
where in the last step we have used the fact that the correlation matrix $\mathbf{C}$ has strictly positive eigenvalues. This proves eq.~(\ref{eq:variance-shrinkage}). Now suppose that $\mathbf{\Sigma}$ is not purely diagonal. Then~$\mathbf{C}$ is also not diagonal; hence there must be some unit vector~$\mathbf{e}_i$ that is not an eigenvector of $\mathbf{C}$. In this case the inequality in eq.~(\ref{eq:eigen-ineq}) is strict, showing that $\Sigma_{ii}>\Psi_{ii}$.
\end{proof}

Next we examine the difference in entropy given by eq.~(\ref{eq:entropy-loss}). First we show that this difference is determined by the trade-off of competing entropic forces.

\begin{theorem}[The Shrinkage-Delinkage Tradeoff]
Consider the entropy difference in eq.~(\ref{eq:entropy-loss}) from FG-VI:
     \begin{equation*}
        \mathcal H(p) - \mathcal H(q) = \tfrac{1}{2} \log|\mathbf{S}| - \tfrac{1}{2} \log |\mathbf{C}|^{-1}.
      \end{equation*}
      Both terms in this difference are nonnegative: that is,
      \begin{align}
      \log |\mathbf{S}|\ &\geq\ 0 \label{eq:logS},\\
      \log \frac{1}{|\mathbf{C}|} &\geq\ 0. \label{eq:logC}
      \end{align}
      \label{thm:trade-off}
  \end{theorem}

\vspace{-3ex}
Before proving the theorem we consider the meaning of these inequalities. Conceptually, the first inequality shows that any shrinkage of variances (from Theorem~\ref{thm:shrinkage-variance}) reduces the entropy of $q$ and thus contributes to a larger difference in eq.~(\ref{eq:entropy-loss}). The second inequality shows that the factorization of $q$ acts as a counterbalance to this effect: the entropy of~$p$ is necessarily reduced by the presence of correlations, but such correlations cannot be modeled by~$q$. Thus the factorization of $q$ must (to some extent) oppose the entropy difference in eq.~(\ref{eq:entropy-loss}), and the net difference is determined by the trade-off of these forces. Visually the factorization of $q$ is represented by the delinkage of nodes in the full-covariance graphical model for $p$. This is the essence of the {\it shrinkage-delinkage}~tradeoff for FG-VI.

\begin{proof}
The bound on $\log|\mathbf{S}|$ in eq.~(\ref{eq:logS}) follows at once from Theorem~\ref{thm:shrinkage-variance}:
  \begin{equation}
    \log|\mathbf{S}| = \sum_{i = 1}^n \log \frac{\Sigma_{ii}}{\Psi_{ii}} \ge 0.
  \end{equation}
As before, let $\lambda_1,\ldots,\lambda_n$ denote the eigenvalues of $\mathbf{C}$ so that $\log|\mathbf{C}| = \sum_i \log\lambda_i$. From Jensen's inequality, we~have:
\begin{equation}
\sum_{i=1}^n \log \lambda_i
    \leq n\log\bigg[\tfrac{1}{n}\sum_{i=1}^n\lambda_i\bigg]
    = n\log\tfrac{1}{n}\,\text{trace}(\mathbf{C}) = 0,
\end{equation}
which proves eq.~(\ref{eq:logC}).
\end{proof}
To prove that $q$ underestimates the entropy of $p$, we need the following result which is important in its own right.

\begin{proposition}
The entropy gap between $p$ and $q$ is equal to the KL divergence minimized by FG-VI:
    \begin{equation}
    \label{eq:gapKL}
        \mathcal H(p) - \mathcal H(q) = {\rm KL}(q||p). 
    \end{equation}
\end{proposition}
\begin{proof}
    The identity follows by substituting the solution from eq.~(\ref{eq:Psi}) into the KL divergence in eq.~(\ref{eq:KL}). This yields the entropy gap, namely ${\rm KL}(q,p) = \frac{1}{2}\log|\boldsymbol\Sigma| - \frac{1}{2}\log|\boldsymbol{\Psi}|$,  computed in eq.~(\ref{eq:entropy-loss2}).
\end{proof}

It follows that FG-VI is minimizing the entropy gap between $p$ and $q$ when it targets a multivariate Gaussian. As suggested by the trade-off in Theorem~\ref{thm:trade-off}, however, the entropy gap can be minimized despite a large shrinkage in componentwise variances. 

The nonnegativity of the KL divergence in eq.~(\ref{eq:gapKL}) also leads to the other fundamental inequality of this section.
\begin{theorem}[Entropy gap] \label{thm:shrinkage-entropy}
The solution for FG-VI in eq.~(\ref{eq:Psi}) underestimates the entropy; that is,
\begin{equation}
\mathcal{H}(q) \leq \mathcal{H}(p),
\label{eq:entropy-shrinkage}
\end{equation}
and this inequality is strict if $\mathbf{\Sigma}$ is not purely diagonal.
\end{theorem}
An immediate implication of this theorem is that the shrinkage term in eq.~(\ref{eq:logS}) dominates the shrinkage-delinkage trade-off in Theorem~\ref{thm:trade-off}.

\begin{remark}
We see also from Theorem~\ref{thm:shrinkage-entropy} that $|\mathbf{\Psi}| \leq |\mathbf{\Sigma}|$. This inequality can be viewed as a multivariate analog of the result, in Theorem~\ref{thm:shrinkage-variance}, that $\Psi_{ii} \leq \Sigma_{ii}$.
\end{remark}

\subsection{Demonstration of the trade-off}
\label{sec:corr}

Figure~\ref{fig:entropy_loss} illustrates the shrinkage-delinkage trade-off in \mbox{FG-VI}, and how it is resolved, for multivariate Gaussian distributions with two types of covariance matrices:
 \begin{itemize}
 
       \item \textit{Squared exponential kernel:} 
       This type of covariance matrix arises in models involving Gaussian processes~\citep[e.g.][Chapter 2]{Rasmussen:2006}. For this example we sampled a random input ${\bf x} \sim \text{uniform}(0, 200)^n$ and set the covariance matrix via the kernel function 
       \mbox{$\Sigma_{ij} = \exp(-(x_i\!-\!x_j)^2/\rho^2)$}.
      We use the hyperparameter $\rho\!>\!0$ to vary the degree of correlation.
      
      \item \textit{Constant off-diagonal:} 
      The posterior distributions of models with exchangeable data~\citep[chapter 5]{Gelman:2013} can generate such covariance matrices, or at least covariance matrices whose subblocks have the described structure. For this example we set $\Sigma_{ii}\!=\!1$ along the diagonal and $\Sigma_{ij}\!=\!\varepsilon$ for all $i\!\neq\!j$, and we used the hyperparameter $\varepsilon\!>\!0$ to vary the degree of correlation.
  \end{itemize}

\begin{figure}[!h]
      \centering
      \includegraphics[width = 2.5in]{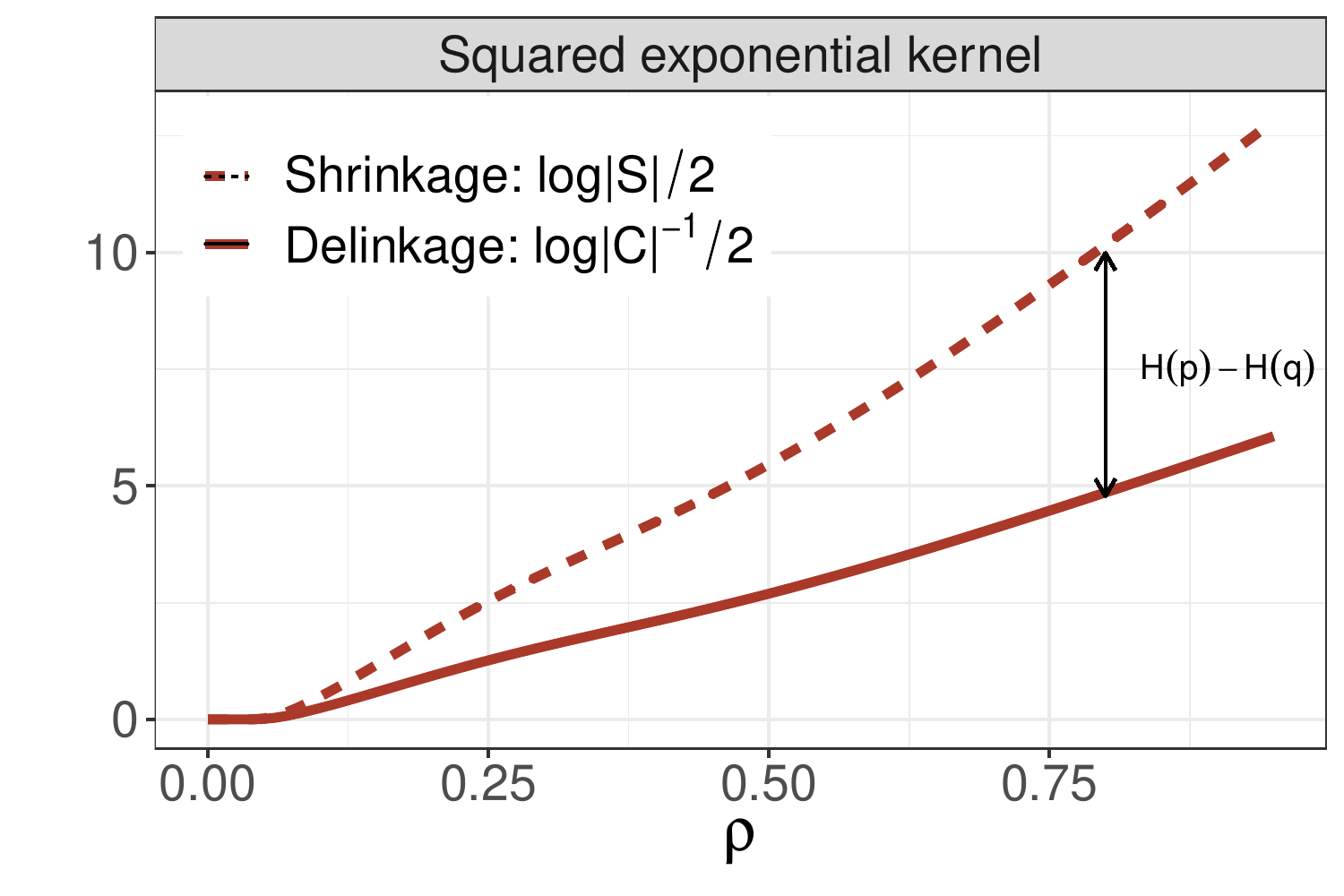}
      \includegraphics[width = 2.5in]{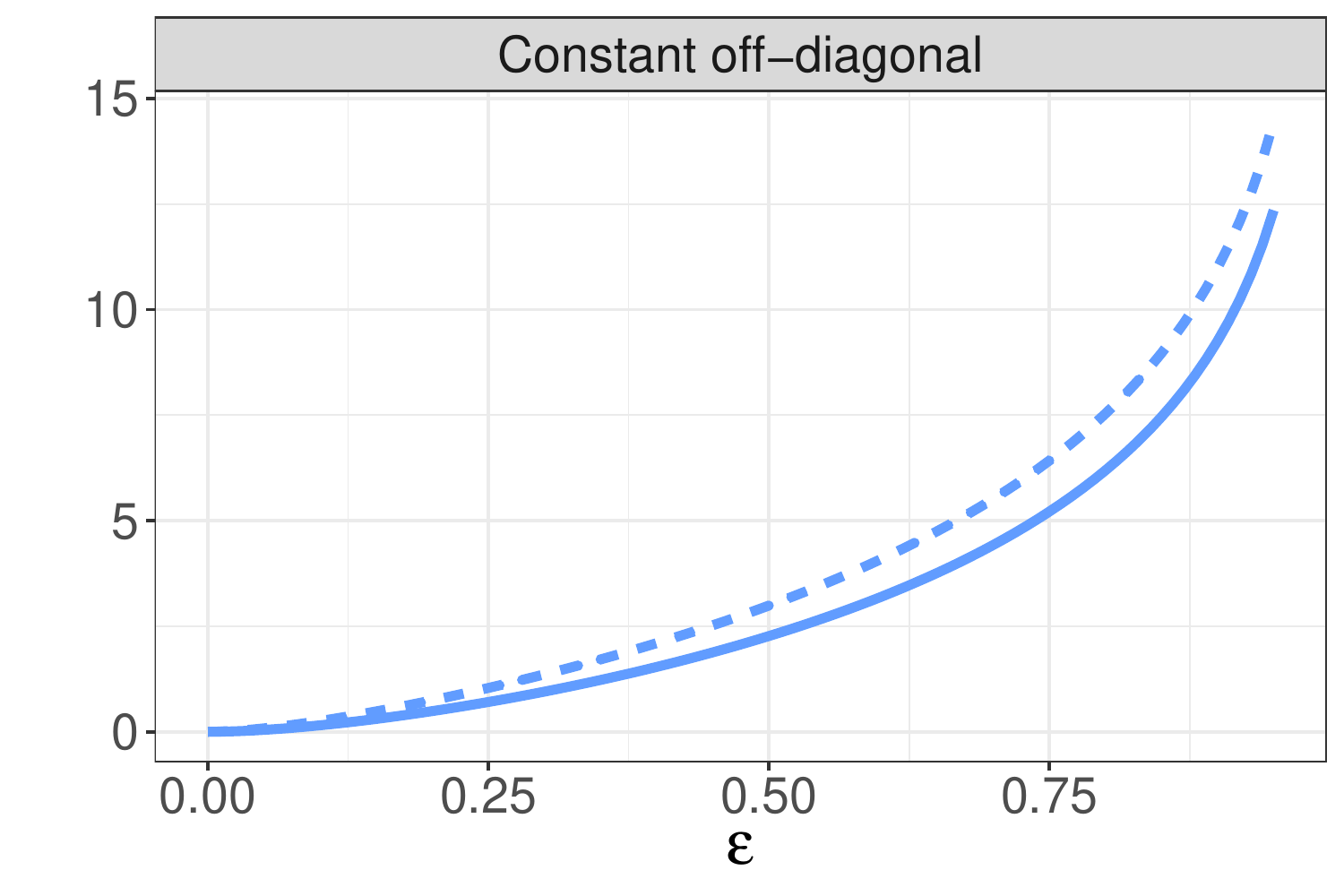}
      \caption{\textit{Shrinkage-delinkage trade-off in FG-VI when the Gaussian target over $\mathbb{R}^n$ has a squared-exponential-kernel covariance matrix ({\it top}) or a covariance matrix with constant off-diagonal terms ({\it bottom}). Here $n\!=\!10$.}}
      \label{fig:entropy_loss}
 \end{figure}
    Figure~\ref{fig:entropy_loss} plots the opposing contributions from the shrinkage and delinkage terms in eq.~(\ref{eq:entropy-loss}) using solid and dashed lines. In each panel, the difference between these curves reveals the degree to which FG-VI underestimates the entropy of the multivariate Gaussian distribution it is being used to approximate. 
    It can also be seen that FG-VI manages the shrinkage-delinkage trade-off differently for different types of covariance matrices. 
    While in the squared exponential kernel case the entropy gap is large, it is smaller when the covariance matrix has constant off-diagonal terms: there, the shrinkage and delinkage terms in eq.~(\ref{eq:entropy-loss}) are almost perfectly balanced.

    This last finding may come as a surprise in light of earlier results, shown in Figure~\ref{fig:shrinkage}, where the variational approximation is clearly too ``compact.'' But the two-dimensional projections in Figure~\ref{fig:shrinkage} are misleading.
    In higher dimensions, the approximating sphere of FG-VI gains more in volume than its target ellipse; this discrepancy arises because each added component is independent for $q$ but strongly correlated for~$p$. The overall effect is that the opposing terms in eq.~(\ref{eq:entropy-loss}) are nearly balanced. Hence even when FG-VI hardly underestimates the (per-component) entropy, it may still grossly underestimate the componentwise variance.
    This contrast becomes more acute in the asymptotic limit of $n$.
    \begin{theorem}  \label{thm:constant-off-diag}
        Suppose $\boldsymbol\Sigma$ has constant off-diagonal terms, $\varepsilon>0$.
        Then the per-component entropy gap vanishes in the limit $n\rightarrow\infty$, whereas every componentwise variance shrinks by a constant factor:
        \begin{align}
            \underset{n \to \infty} \lim \ \tfrac{1}{n} \left(\mathcal H(p) - \mathcal H(q)\right) &= 0. \\
            \underset{n \to \infty} \lim \ (\Psi_{ii}/\Sigma_{ii}) &= 1 - \varepsilon.
        \end{align}
      \end{theorem}
    The proof is given in Appendix~A. The theorem also shows that the average of the diagonal elements in the shrinkage matrix also converges to a constant factor:
        \begin{equation}
            \underset{n \to \infty} \lim \ \tfrac{1}{n} \mathrm{trace}({\bf S}) = (1 - \varepsilon)^{-1}.
        \end{equation}
        
    This example highlights the roots of FG-VI in mean-field approximations from statistical physics~\citep{Parisi:1988}. As is well known, the mean-field approximation for the free energy becomes exact in the limit $n\!\rightarrow\!\infty$ for certain spin systems with infinite-range interactions. The infinite-range interactions in these systems are analogous, for the Gaussian models we study here, to the assumption of constant off-diagonal terms in the covariance matrix~\citep[e.g][]{Mukherjee:2018, Margossian:2021}.
    An important takeaway is that factorized approximations can work well to estimate the entropy---to wit, minimizing the KL-divergence with FG-VI on a Gaussian target is equivalent to minimizing the entropy gap---but still fail to accurately compute the componentwise variances.
    This can become an important limitation as we apply VI beyond problems in statistical physics and more broadly to Bayesian modeling, where estimation of the variances is critical.

    FG-VI's limited ability to estimate the marginal variance is a product of both the choice of the approximating family (factorized Gaussians) and the choice the objective function, $\text{KL}(q || p)$.
    In Appendix~A we show that when minimizing the \textit{reverse} KL-divergence, $\text{KL}(p || q)$, for a target $p$, we obtain, in the above example, the opposite result: exact estimations of the marginal variances but an arbitrarily large entropy gap.


\section{Bounds on ${\bf log}|\mathbf{S}|$ and ${\bf log}|\mathbf{C}|$}

In the last section, we saw that the shrinkage-delinkage trade-off played out differently for different types of covariance matrices; we also proved certain asymptotic results that depended on the detailed structure of the covariance matrix (e.g., constant off-diagonal). 
In this section, we derive more general bounds on the terms in this trade-off that depend only
the problem dimensionality, $n$, and the condition number, $R$, of the correlation matrix, ${\bf C}$.


\subsection{Optimizations for upper bounds}

Consider the space of all correlation matrices with condition number $R$. We denote this space by the set
\begin{equation}
    \label{eq:CR}
    \mathcal{C}_R = \{\mathbf{C}\in\mathcal{S}^{n}_+\, |\, C_{ii}\!=\!1\ \forall i,\, \lambda_{\rm max}(\mathbf{C})\!=\! R\lambda_{\rm min}(\mathbf{C})\}.
\end{equation}
The set contains the intersection of those $n\!\times\! n$ matrices that are positive semidefinite (i.e., lying in the cone $\mathcal{S}_+^n$), whose diagonal elements are equal to unity, and whose largest eigenvalue is $R$ times larger than its smallest one.

If the condition number of the correlation matrix $\mathbf{C}$ is known to be $R$, then we can (in principle) compute the following upper bounds on the terms in eq.~(\ref{eq:entropy-loss}):
\begin{align}
\log|\mathbf{S}| &\leq \max_{\mathbf{C}\in\mathcal{C}_R}\left[\sum_{i=1}^n \log C_{ii}^{-1}\right],
\label{eq:boundS1} \\
\log|\mathbf{C}| &\leq \max_{\mathbf{C}\in\mathcal{C}_R}\left[\sum_{i=1}^n \log \lambda_i(\mathbf{C})\right].
\label{eq:boundC1}
\end{align} 
In eq.~(\ref{eq:boundS1}), we have used the fact from eq.~(\ref{eq:Sii}) that \mbox{$S_{ii} \!=\! C^{-1}_{ii}$}, while in eq.~(\ref{eq:boundC1}), we have written the determinant of a matrix as the product of its eigenvalues.

In practice, however, it is difficult to perform the optimizations over the set $\mathcal{C}_R$ in eq.~(\ref{eq:boundS1}-\ref{eq:boundC1}). Instead we consider a more tractable relaxation; the essential idea is to optimize over a larger set of matrices, one that is characterized only in terms of constraints on its eigenvalues. We denote this constrained set of eigenvalues by
\begin{equation}
\label{eq:LR}
\hspace{-1.75ex}\Lambda_R = \left\{\boldsymbol{\lambda}\!\in\!\mathbb R^n_+\, \bigg|\, \lambda_1\!\geq\!\ldots\geq\!\lambda_n=R\lambda_1,\, \sum_{i=1}^n \lambda_i \!=\! n\right\}\!.
\end{equation}
Note that the set ${\cal C}_R$ of correlation matrices is contained strictly within the set of matrices with eigenvalues in $\Lambda_R$. In particular, a matrix in ${\cal C}_R$ is constrained to have ones along its diagonal, while a matrix with eigenvalues in $\Lambda_R$ is only constrained to have a trace equal to $n$. With the above relaxation, we obtain the following upper bounds
on the terms $\log|\mathbf{S}|$ and $\log|\mathbf{C}|$ in eq.~(\ref{eq:entropy-loss}). 

\begin{proposition} 
Suppose that the correlation matrix $\mathbf{C}$ has condition number $R$. Then
\begin{align}
\log |\mathbf{S}| &\leq n \log \frac{1}{n}\!\left[\max_{\boldsymbol\lambda\in\Lambda_R}
  \sum_{i=1}^n \lambda_i^{-1}\right]
\label{eq:boundS2} \\
\log |\mathbf{C}| &\leq \max_{\boldsymbol\lambda\in\Lambda_R}\left[\sum_{i=1}^n \log \lambda_i\right].
\label{eq:boundC2}
\end{align}
\end{proposition}
\begin{proof}
The second bound is immediate from eq.~(\ref{eq:boundC1}) and the relaxation in eq.~(\ref{eq:LR}). 
For the first bound, recall that $S_{ii} = C^{-1}_{ii}$, and note from Jensen's equality that
\begin{equation*}
\tfrac{1}{n}\mbox{$\sum_i$} \log C_{ii}^{-1} \leq \log\tfrac{1}{n}\mbox{$\sum_i$}\, C^{-1}_{ii} = \log\left[\tfrac{1}{n}\mbox{$\sum_i$}\, \lambda_i^{-1}(\mathbf{C})\right].
\end{equation*}
The bound in eq.~(\ref{eq:boundS2}) follows from the above in concert with the relaxion in eq.~(\ref{eq:LR}).
\end{proof}


\subsection{Solutions from symmetry}

\begin{figure*}
    \includegraphics[width=3in]{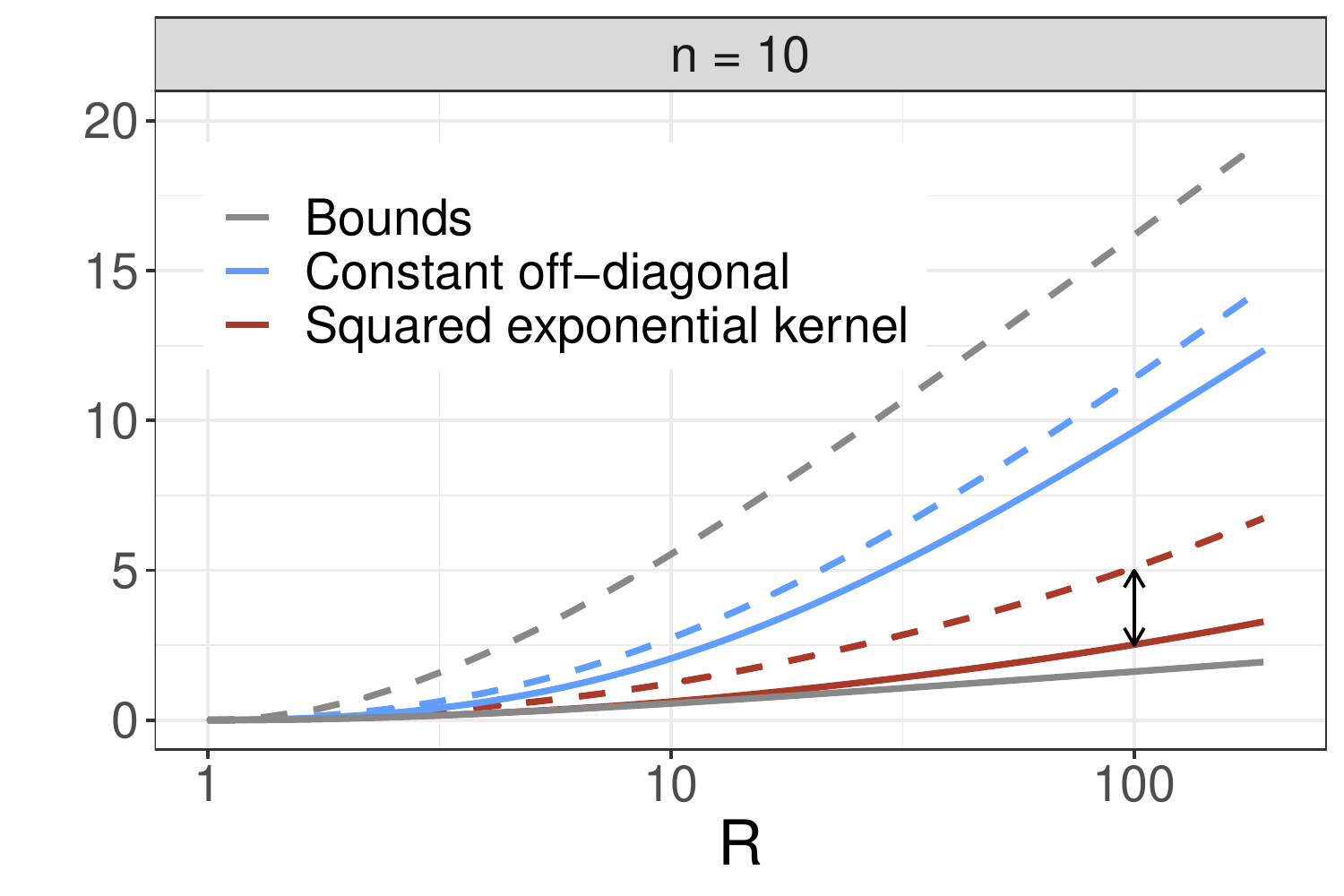}
    \includegraphics[width=3in]{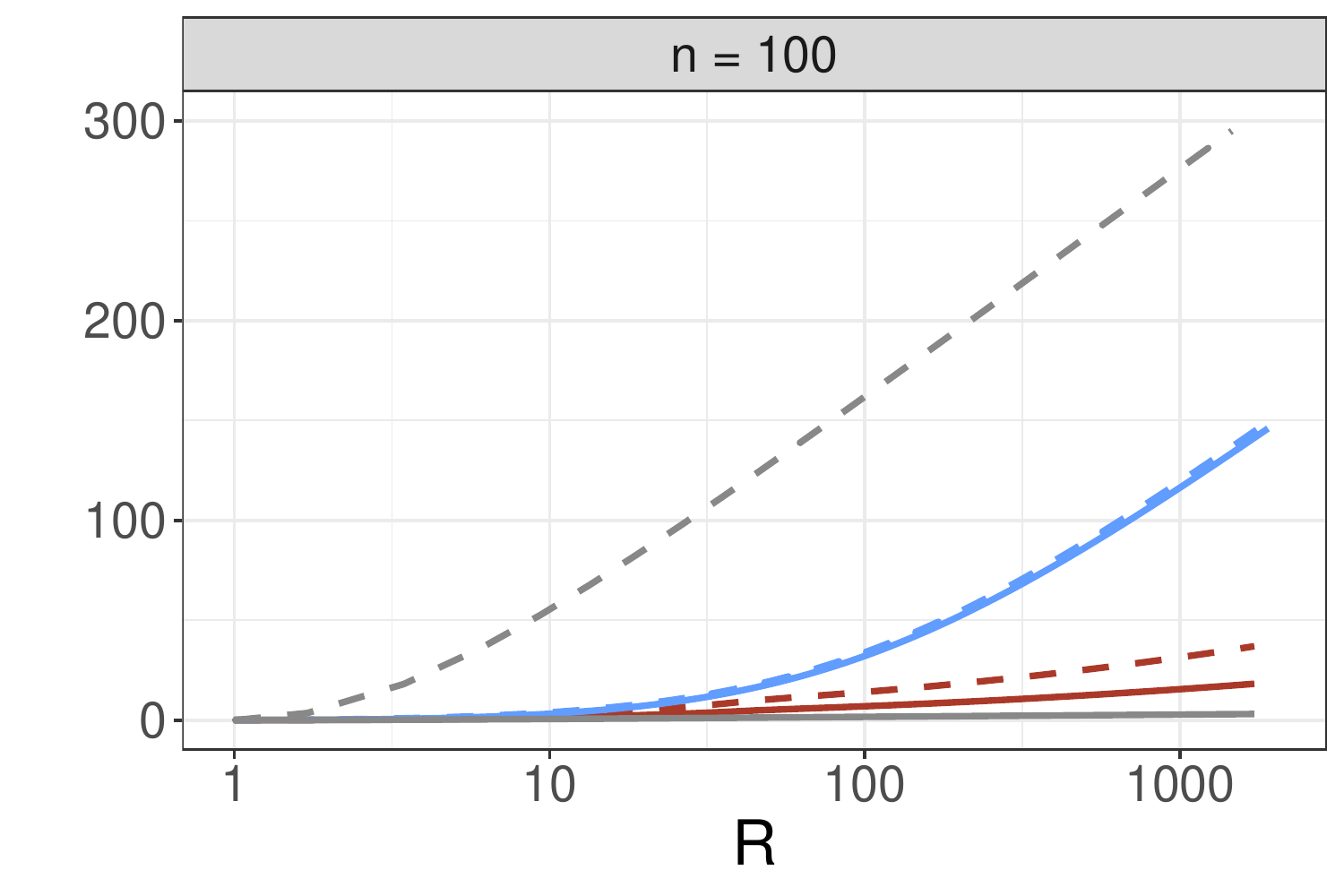}
    \caption{\textit{
    Bound on the entropy gap.
    The blue and red curves are replotted from Figure~\ref{fig:entropy_loss}.
    The dotted gray line is an upper bound on the shrinkage term, $\log |S| / 2$, and the solid gray line is a lower bound on the delinkage term, $\log|C|^{-1} / 2$.
    Hence the difference between the gray lines provides an upper bound on the entropy gap, $\mathcal H(p) - \mathcal H(q) = log |S| / 2 - \log|C|^{-1} / 2$ (Theorem~\ref{thm:trade-off}).
    }}
    \label{fig:entropy_bound}
\end{figure*}

The optimizations over $\Lambda_R$ in eqs.~(\ref{eq:boundS2}--\ref{eq:boundC2}) have a great deal of structure that we can exploit to compute their solutions. We analyze each of these optimizations in turn.

\begin{lemma} \label{lemma:symmetry}
  Let $\boldsymbol\lambda\in\Lambda_R$ be the solution that maximizes the right side of eq.~(\ref{eq:boundS2}). Then at most one $\lambda_i$ is not equal to either $\lambda_1$ or~$\lambda_n$.
\end{lemma}
\begin{proof}
We prove the lemma by contradiction. Suppose there exists a solution with intermediate elements $\lambda_i$ and $\lambda_j$ that satisfy
$\lambda_1\! >\! \lambda_i\! >\! \lambda_j\! >\! \lambda_n$.
Consider the effect on this solution of a perturbation that adds some small amount $\delta\!>\!0$ to $\lambda_i$ and subtracts the same amount from $\lambda_j$. Note that for sufficiently small~$\delta$, this perturbation will not leave the set~$\Lambda_R$; however, it will {\it expand} the separation of $\lambda_i$ from~$\lambda_j$. As a result the objective $\sum_i \lambda_i^{-1}$ has a gain
\begin{equation}
  f(\delta) = \frac{1}{\lambda_i+\delta} - \frac{1}{\lambda_i} + \frac{1}{\lambda_j-\delta} - \frac{1}{\lambda_j}.
\end{equation}
Next we evaluate the derivative $f'(\delta)$ at $\delta=0$; doing so we find $f'(0) = \lambda_j^{-2}\! -\! \lambda_i^{-2} > 0$. But this yields a contradiction, because any solution must be maximal, and hence stationary (i.e., $f'(0)\!=\!0$), with respect to small perturbations.
\end{proof}

The above lemma greatly restricts the form of the solutions that we must consider for the optimization in eq.~(\ref{eq:boundS2}). The next lemma does the same for the optimization in eq.~(\ref{eq:boundC2}).

\begin{lemma} \label{lemma:symmetry2}
  Let $\boldsymbol\lambda\in\Lambda_R$ be the solution that maximizes the right side of eq.~(\ref{eq:boundC2}). Then $\lambda_i\!=\!\lambda_j$ whenever $1\!<\!i\!<\!j\!<\!n$.
\end{lemma}
\begin{proof}
We prove this lemma in similar fashion. Suppose there exists a solution
with intermediate elements $\lambda_i$ and $\lambda_j$ that satisfy
$\lambda_1\! \geq\! \lambda_i\! >\! \lambda_j\! \geq\! \lambda_n$.
Consider the effect on this solution of a perturbation that adds some small amount $\delta\!>\!0$ to $\lambda_j$ and subtracts the same amount from~$\lambda_i$. Again, for sufficiently small~$\delta$, this perturbation will not leave the set~$\Lambda_R$; however, it will {\it diminish} the separation of~$\lambda_i$ from~$\lambda_j$. As a result the objective $\sum_i \log\lambda_i$ has a gain
\begin{equation}
  g(\delta) = \log(\lambda_i-\delta) +\log(\lambda_j+\delta) -\log(\lambda_i) - \log(\lambda_j).
  \end{equation}
Evaluating the derivative, we find $g'(0) = \lambda_j^{-1}\! -\! \lambda_i^{-} > 0$. As before this yields a contradiction, because any solution must be maximal, and hence stationary (i.e., $g'(0)\!=\!0$), with respect to small perturbations.
\end{proof}

Now let us consider, at a high level, how these lemmas simplify the optimizations in eqs.~(\ref{eq:boundS2}--\ref{eq:boundC2}).
The lemmas show that for each optimization, there exist {\it three} elements---the maximum element $\lambda_1$, the minimum element $\lambda_n$, and some intermediate element $\lambda_k$ for $1\!<\!k\!<\!n$---from which the remaining $n\!-\!3$ elements of the solution can be deduced by symmetry. Note also that any solution in $\Lambda_R$ must satisfy the {\it two additional} constraints that $\sum_i\lambda_i\!=\!n$ and $\lambda_1\!=\!R\lambda_n$. 
In Appendix B, we show that by exploiting these symmetries and constraints in concert, we can reduce the optimizations in eqs.~(\ref{eq:boundS2}--\ref{eq:boundC2}) to a sequence of one-dimensional problems for which we have closed-form solutions. 

    Figure~\ref{fig:entropy_bound} plots the bounds on the shrinkage and delinkage terms as a function of the condition number, $R$, for problems with dimensionalities $n\!=\!10$ (\textit{left}) and $n\!=\!100$ (\textit{right}). These bounds provide envelopes between which the actual values of the competing terms in eq.~(\ref{eq:entropy-loss}) must lie. To illustrate this, the figure also shows
    the corresponding values of these terms for the squared-exponential-kernel and constant-off-diagonal covariance matrices introduced in the previous section. (Notice that in this figure, unlike Figure~\ref{fig:entropy_loss}, these values are plotted against the condition number of the correlation matrix rather than the hyperparameters $\rho$ or $\varepsilon$.) 

    Using similar methods, it also possible to upper-bound the entropy gap and the trace of the shrinkage matrix (which reflects the average shrinkage in componentwise variance).
    The derivations of these additional bounds are relegated to Appendices~C and D.  


  \section{Non-Gaussian models} \label{sec:examples}

 Do our results extend in any way when FG-VI is applied to non-Gaussian models?
  In this section we suppose that~$p$ is a \textit{non-Gaussian} target with covariance~$\boldsymbol\Sigma$.
 Our previous analysis of FG-VI targets was based on the variance estimator,
  \begin{equation*}
      (\Psi_G)_{ii} := 1 / (\Sigma^{-1})_{ii},
  \end{equation*}
  and the corresponding shrinkage matrix with diagonal elements $(S_G)_{ii} = \Sigma_{ii} / (\Psi_G)_{ii}$.
  But neither $\mathbf{\Psi}_G$ nor $\mathbf{S}_G$ will be returned by FG-VI when it is applied to a non-Gaussian target with covariance $\boldsymbol\Sigma$. 

To explore these issues, we applied ADVI~\citep{Kucukelbir:2017} with a factorized Gaussian approximation to study the posterior distributions in several Bayesian models as well as one ``adversarial'' target (Table~\ref{tab:models}).
  These test targets represent a diversity of applications.
  The GLM and 8-schools models are taken from the model data base PosteriorDB \citep{Magnusson:2022}, while the disease map Gaussian process model and sparse kernel interaction model \citep{Agrawal:2019} are studied by \citep{Margossian:2020}.
  We also included a mixture of well-separated spherical Gaussians; for this target, the approximation by FG-VI collapses to one of the modes, so that FG-VI can underestimate the componentwise variances by an arbitrarily large amount (e.g., if the modes are widely separated). 
  Note that in all cases, before applying ADVI, we transformed any constrained (e.g., nonnegative) variables of the target distribution to be unconstrained variables over $\mathbb{R}$.

    \begin{table*}[!htb]
  \small
      \centering
      \renewcommand{\arraystretch}{1.1}
      \begin{tabular}{l r l r}
         \rowcolor{Gray} {\bf Call}  &  $n$ & {\bf Description} & $\frac{1}{n} \log|\boldsymbol \Sigma / \boldsymbol \Psi|$ \\
         \texttt{glm\_binomial} & 3 & General linear model with a binomial likelihood. &  0.291\\
         \rowcolor{Gray} \texttt{8schools\_nc} & 10 & Hierarchical model with a non-centered parameterization. & 0.011 \\
         \texttt{8schools\_pool} & 9 & Same model but with a small, fixed population variance value to enforce & 0.339 \\
         & & strong partial pooling and create a high posterior correlation. \\
         \rowcolor{Gray} \texttt{disease\_map} & 102 & Gaussian process model with Poisson likelihood. Applied to disease & 0.066 \\
        \rowcolor{Gray} & & map of Finland using 100 randomly sampled counties (out of 911). & \\
        \texttt{SKIM} & 305 & Sparse kernel interaction model, applied to a Prostate cancer microarray & 0.033 \\
        & & data set on a subset of 200 SNPS. \\
        \rowcolor{Gray} \texttt{Mixture} & 2 & Mixture of well-separated Gaussians with spherical covariance matrices. & 3.051
      \end{tabular}
      \caption{\textit{Non-Gaussian targets for numerical experiments.}}
      \label{tab:models}
  \end{table*}
  
  We estimated the posterior covariance using long runs of Markov chain Monte Carlo,
  specifically 16,000 draws using the software Stan \citep{Carpenter:2017}, except in the mixture example, where the covariance was calculated analytically.
  We then estimated (i) the shrinkage matrix, $\mathbf{S}$, when targeting the posterior and (ii) the shrinkage matrix,~$\mathbf{S}_G$, when targeting a Gaussian with the same covariance as the posterior.
  For non-Gaussian posteriors, we observe that FG-VI does \textit{not} always underestimate the componentwise variance; see Figure~\ref{fig:8schools_shrinkage}.
  On the other hand, for all models in Table~\ref{tab:models}, we see that $\frac{1}{n} \text{trace}(\mathbf{S})\! >\! 1$, meaning that the componentwise variances are underestimated \textit{on average} (Figure~\ref{fig:model_S}).
  In addition, for the Bayesian models, we observe that \mbox{$\text{trace}(\mathbf{S})\! \approx\! \text{trace}(\mathbf{S}_G)$}.
  The mixture target, however, provides a counter-example where $\text{trace}(\mathbf{S}) \gg \text{trace}(\mathbf{S}_G)$.

    \begin{figure}[!h]
      \centering
      \includegraphics[width = 3in]{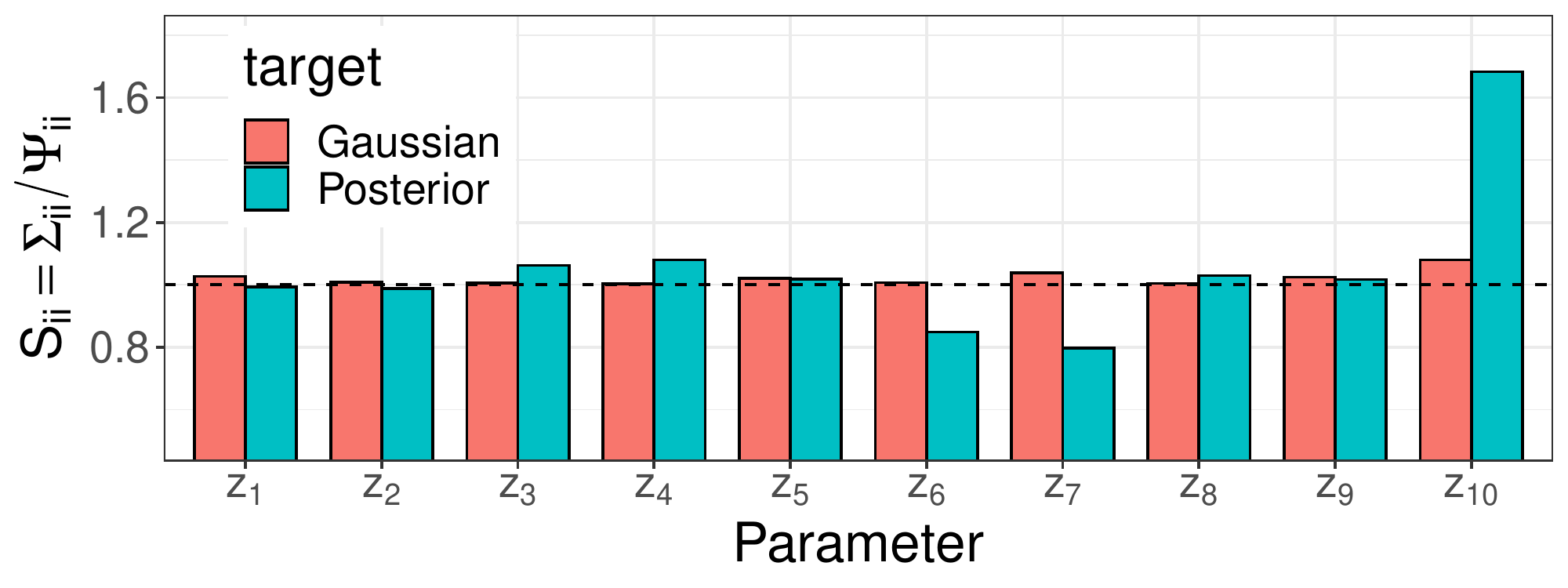}
      \caption{\textit{Shrinkage matrix for FG-VI when targeting the posterior distribution of \texttt{8schools\_nc} versus targeting a Gaussian with the same covariance matrix. For the non-Gaussian target, we may have $S_{ii}\! <\! 1$.}}
      \label{fig:8schools_shrinkage}
  \end{figure}

  \begin{figure}[!h]
    \centering
    \includegraphics[width=3in]{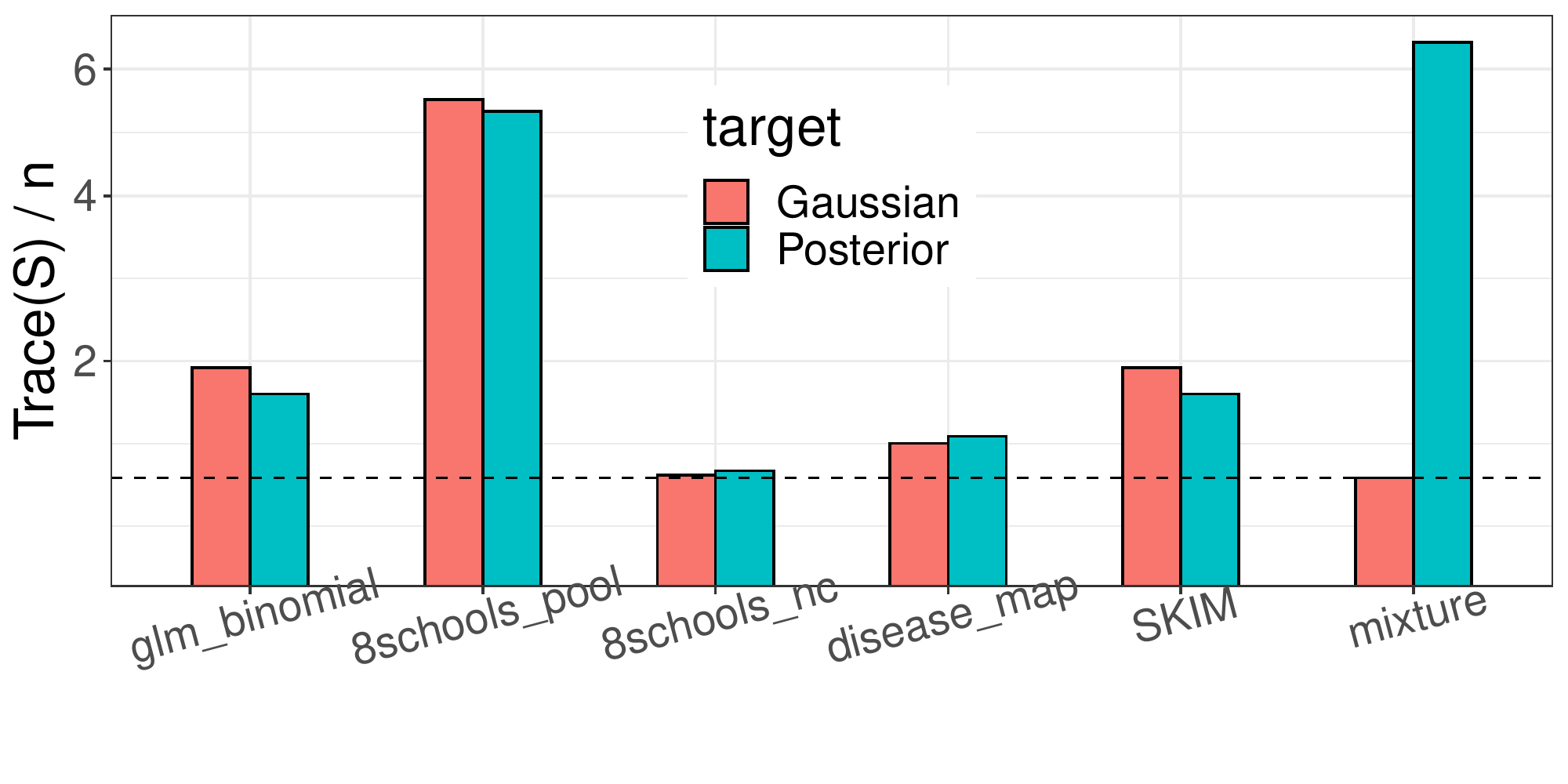}
    \caption{\textit{Trace of shrinkage matrix for various models when targeting the true posterior versus targeting a Gaussian with the same covariance matrix.}}
    \label{fig:model_S}
\end{figure}  

  In addition to the shrinkage in componentwise variances, we would have also liked to evaluate the entropy gap in these models.
  It is easy to obtain an upper bound on this gap by observing that the Gaussian maximizes the entropy among all continuous distributions with a given covariance~$\boldsymbol \Sigma$~\citep{Cover:2006}. Thus we have
  \begin{equation}  \label{eq:non-gaussian-bound}
      \mathcal H(p) - \mathcal H(q) \le \tfrac{1}{2}(\log|\boldsymbol \Sigma| - \log|\boldsymbol \Psi|).
  \end{equation}
   We observed this upper bound to be positive for all the models in Table~\ref{tab:models}. But we also know that FG-VI can overestimate the entropy in non-Gaussian models: \citet{Turner:2011} demonstrated this for a mixture of largely overlapping 1-dimensional Gaussians.
  It is an open question to understand the general conditions under which FG-VI underestimates the entropy.
  We next note that our upper-bound on the entropy gap does not immediately apply to \eqref{eq:non-gaussian-bound}, because it is unclear how $\log |\boldsymbol \Psi|$ and $\log |\boldsymbol \Psi_G|$ compare.
   To evaluate the entropy gap empirically in a non-Gaussian model---as would be required to investigate this question further---it is necessary to estimate the normalizing constant of the posterior.
  Candidate methods for this, such as bridge sampling \citep[e.g][]{Gronau:2020, Meng:2002}, rely on a proposal distribution which (roughly) approximates the target.
  Typically, a Gaussian-like approximation is used for these proposals, but this is precisely the assumption we want to relax. In other words, we do not wish to compare a theory for Gaussian targets to an empirical benchmark which relies on a Gaussian approximation. We leave this issue to future work.

\section{Discussion}

In this paper we have shown that FG-VI underestimates the componentwise variance and joint entropy of a multivariate Gaussian distribution.
Furthermore we expressed the entropy gap as a trade-off between two competing terms and observed that it was equal to the KL divergence minimized by FG-VI.
Our analysis helps to understand why FG-VI can greatly underestimate the componentwise variances even when it effectively minimizes the entropy gap and KL divergence.
Our results also suggest that better estimates of variance may be obtained by changing the objective function or using a different family of approximations.

This research has practical implications on when to use FG-VI. When the target distribution exhibits strong correlations, FG-VI can return poor estimates of the marginal variance; this is a limitation in Bayesian modeling where we often care about the posterior variance of interpretable variables.
On the other hand, FG-VI can still produce good estimates of the entropy, notably in the limit where $n$ is large. This is one reason that factorized approximations have been widely used for many problems in statistical physics.

An open question is whether a shrinkage-delinkage tradeoff operates when FG-VI is applied to non-Gaussian targets.
For such targets, we have produced a counter-example showing that FG-VI can overestimate
a particular componentwise variance. On the other hand, we have observed that these variances are underestimated on average, and moreover that the shrinkage term, $\log |{\bf S}|$, remains positive.
It requires further investigation to make more general statements about the entropy gap.
Finally, it would also be interesting to extend our analysis beyond FG-VI---for instance to approximations based on a diagonal plus low-rank covariance matrix, rather than a strictly diagonal one \citep[e.g][]{Zhang:2022}.

\section{Acknowledgments}

We thank David Blei, Bob Carpenter, Justin Domke and Chirag Modi for feedback on this manuscript.

\bibliography{ref}

\newpage
\
\newpage

\appendix

\section*{Appendix}

In appendix~\ref{app:cov}, we provide further details on FG-VI when the covariance matrix has constant off-diagonal terms. Next appendix~\ref{app:alg} provides an algorithm to efficiently compute bounds on the shrinkage and delinkage terms, using techniques developped in section 4.
In appendix~\ref{app:traceS}, we show how to derive bounds on the trace of the shrinkage matrix.
In appendix~\ref{app:KL}, we obtain an upper bound on the KL divergence---equivalently the entropy gap---between $q$ and~$p$.


\appendix
\section{FG-VI for covariance with constant off-diagonal terms}
\label{app:cov}

Let $p$ be a multivariate Gaussian distribution over $\mathbb{R}^n$ with mean $\boldsymbol \mu$ and covariance matrix $\boldsymbol\Sigma$.
The elements of the correlation matrix are related to those of the covariance matrix by
\begin{equation}
    C_{ij}=\frac{\Sigma_{ij}}{\sqrt{\Sigma_{ii}\Sigma_{jj}}}.
\label{eq:corr}
\end{equation} 
In this section we assume that the correlation matrix has constant off-diagonal terms, and we use $\varepsilon\in[0,1)$ to denote the value of these terms. Note that we require $\varepsilon\geq 0$ since three or more random variables cannot all be mutually anti-correlated. Also we require $\varepsilon\!<\!1$ since otherwise $\mathbf{C}$ (and hence $\boldsymbol\Sigma$) would not be positive-definite.

\subsection{Solution for factorized Gaussian variational inference}

Let $q$ be the solution of FG-VI with diagonal covariance matrix $\boldsymbol\Psi$, where
\begin{equation}
    \Psi_{ii} = \frac{1}{\Sigma^{-1}_{ii}}
\label{eq:psi}
\end{equation} as in eq.~(4).

We now prove Theorem~3.5, broken up into a statement about the estimated variance (Proposition~\ref{prop:asymptotic-shrinkage}) and a statement about the entropy gap (Proposition~\ref{prop:asymptotic-entropy}).

\begin{proposition} \label{prop:asymptotic-shrinkage}
    If the correlation matrix in eq.~(\ref{eq:corr}) has constant off-diagonal terms, then the solution for FG-VI in eq.~(\ref{eq:psi}) obeys the following limits:
\begin{eqnarray}
   \underset{n \to \infty}{\lim} \Psi_{ii} & = & (1\! -\! \varepsilon)\Sigma_{ii} \\
    \underset{\varepsilon \to 1}{\lim} \Psi_{ii} & = & 0 \\
    \underset{n \to \infty}{\lim} \frac{1}{n} \mathrm{trace}({\bf S}) & = & \frac{1}{1 - \varepsilon}
 \end{eqnarray}
 where $\varepsilon\in[0,1)$ denotes the value of $C_{ij}$ for $i\neq j$.
\end{proposition}

\begin{proof}
Let $\mathbf{1}\in\mathbb{R}^n$ denote the vector of all ones. Then the correlation matrix can be written as
 \begin{equation}
      \mathbf{C} = (1 - \varepsilon) \mathbf{I} + \varepsilon {\bf 1} {\bf 1}^\top.
      \label{eq:corr-epsilon}
  \end{equation}
  One can verify by direct substitution that the inverse correlation matrix has elements
  \begin{equation}
      \mathbf{C}^{-1} = \tfrac{1}{1 - \varepsilon}\left[{\bf I} - \tfrac{\varepsilon}{1 + (n - 1) \varepsilon} {\bf 1}{\bf 1}^\top\right].
   \label{eq:Cinv}
  \end{equation}
Recall from eq.~(15) that $\Psi_{ii} = \Sigma_{ii}/C_{ii}^{-1}$. With some algebra, it follows from eq.~(\ref{eq:Cinv}) that
\begin{equation}
    \Psi_{ii} = \left[\frac{(1-\varepsilon)(1+(n\!-\!1)\varepsilon)}{1+(n\!-\!2)\varepsilon}\right]\Sigma_{ii}.
    \label{eq:psi-epsilon}
\end{equation}
It is straightforward to take the limits of eq.~(\ref{eq:psi-epsilon}) as $n\rightarrow\infty$ or $\varepsilon\rightarrow 1$, and these limits yield the results of the proposition.

Finally, recalling that $\text{trace}({\bf S}) = \sum_{i = 1}^n \Sigma_{ii} / \Psi_{ii}$, we immediately get
\begin{equation}
    \underset{n \to \infty}{\lim} \frac{1}{n} \text{trace}({\bf S}) = \frac{1}{1 - \varepsilon}.
\end{equation}
\end{proof}

 \begin{proposition}  \label{prop:asymptotic-entropy}
     If the correlation matrix in eq.~(\ref{eq:corr}) has constant off-diagonal terms, then the per-component entropy gap from FG-VI vanishes in the limit $n\rightarrow\infty$; that is,
      \begin{equation}
          \underset{n \to \infty}{\lim} \ \frac{1}{n} \left [ \mathcal H(p) - \mathcal H(q) \right ] = 0.
      \end{equation}
  \end{proposition}

\begin{proof}
Recall from Theorem 3.2 of the main paper that the entropy gap for FG-VI is given by
\begin{equation}
    \mathcal{H}(p)-\mathcal{H}(q) = \tfrac{1}{2}\log|\mathbf{S}| + \tfrac{1}{2}\log|\mathbf{C}|,
    \label{eq:gap}
\end{equation}
where $\mathbf{S}$ is the diagonal shrinkage matrix with elements $S_{ii} = \Sigma_{ii}/\Psi_{ii}$. We consider each term on the right side of this equation in turn. It follows at once from eq.~(\ref{eq:psi-epsilon}) that
\begin{equation}
\log|\mathbf{S}| = n\left[\log\frac{1+(n\!-\!2)\varepsilon}{(1\!-\!\varepsilon)(1+(n\!-\!1)\varepsilon)}\right],
\label{eq:logS}
\end{equation}
where $\varepsilon\!>\!0$ denotes the amount of off-diagonal correlation. Next we show how to evaluate $\log|\mathbf{C}|$. From eq.~(\ref{eq:corr-epsilon}), we rewrite the correlation matrix as
   \begin{equation}
        \mathbf{C} = (1\! -\! \varepsilon) {\bf I} + n \varepsilon \left (\tfrac{1}{\sqrt n} {\bf 1} \right) \left (\tfrac{1}{\sqrt n} {\bf 1} \right)^T.
        \label{eq:eigC}
    \end{equation}
    Note that the second term on the right side of eq.~(\ref{eq:eigC}) is a rank-one matrix whose one nonzero eigenvalue is equal to $n\varepsilon$. By adding the first term---which is a multiple of the identity matrix---we obtain a new matrix whose eigenvalues are shifted by a uniform amount. It follows that this new matrix (namely, $\mathbf{C}$) has $n\!-\!1$ eigenvalues at $1\!-\!\varepsilon$ and one eigenvalue at $1+(n\!-\!1)\varepsilon$, so that
    \begin{equation}
        \log|\mathbf{C}| = (n\!-\!1)\log(1\!-\!\varepsilon) + \log (1 + (n\!-\!1) \varepsilon).
        \label{eq:logC}
    \end{equation}
The entropy gap is related to the sum of $\log|\mathbf{S}|$ and $\log|\mathbf{C}|$ by eq.~(\ref{eq:gap}). Adding the results in eq.~(\ref{eq:logS}) and eq.~(\ref{eq:logC}), we find that
\begin{align}
\log|\mathbf{S}| + \log|\mathbf{C}|
    &=  -\log(1\! -\! \varepsilon) + \log (1 + (n\! -\! 1) \varepsilon) \nonumber \\
    &\mbox{\hspace{2ex}}\hspace{3ex} +\ n \log \left [ \tfrac{1 + (n - 2) \varepsilon}{1 + (n - 1) \varepsilon} \right].
    \label{eq:sum}
\end{align}
Note that the first term on the right side is $\mathcal O(1)$, the second term is $\mathcal O(\log n)$, and the third term can be 
written as
    \begin{equation}
        n \log \left [ \tfrac{1 + (n - 2) \varepsilon}{1 + (n -1) \varepsilon} \right]
          = n \log \left [ 1 - \tfrac{\varepsilon}{1 + (n - 1) \varepsilon}  \right].
    \end{equation}
    For large $n$, the log term in this equation is $\mathcal O(\frac{1}{n})$ so that the entire expression is $\mathcal O(1)$. From eq.(\ref{eq:sum}), it therefore follows that the entropy gap in eq.~(\ref{eq:gap}) is $\mathcal O(\log n)$. Dividing by $n$, we see that the per-component entropy gap vanishes in the limit $n\rightarrow\infty$, thereby completing the proof.
  \end{proof}

  From \eqref{eq:sum}, we also see that the entropy gap becomes infinite as $\varepsilon \to 1$ (for fixed $n$).
  Additional limits can be considered with respect to both $\varepsilon$ and $n$, but we do not pursue those here.

  \subsection{Solution when minimizing the reverse KL-divergence}

  A factorized approximation cannot both match the marginal variances and the entropy of the target distribution.
  In the example at hand, minimizing $\text{KL}(p||q)$ leads to good estimates of the entropy but not of the marginal variances.
  We now show that when minimizing the reverse KL-divergence, $\text{KL}(p||q)$, the opposite behavior occurs.

  We first state the solution obtained when minimizing the reverse KL-divergence.
  The following is the counterpart to Proposition~\ref{prop:solution} in the main body of the paper.
  
  \begin{proposition} \label{prop:solution-reverse}
    Let q({\bf z}) be multivariate Gaussian with mean $\tilde{\boldsymbol \nu}$ and diagonal covariance $\tilde{\boldsymbol \Psi}$.
    Then the variational parameters minimizing the reverse KL-divergence, $\text{KL}(p||q)$ are given by $\tilde{\boldsymbol \nu} = \boldsymbol \mu$ and
    \begin{equation}
        \tilde{\Psi}_{ii} = \Sigma_{ii}.
    \end{equation}
  \end{proposition}

  \begin{proof}
      The variational parameters $\tilde{\boldsymbol \nu}$ and $\tilde{\boldsymbol \Psi}$ are estimated by minimizing the reverse KL-divergence
      \begin{equation}
          KL(p||q) = \mathbb E_p [\log p({\bf z})] - \mathbb E_p [\log q({\bf z})],
      \end{equation}
      where each expectation is taken with respect to the measure $p$.
      The first term on the R.H.S does not depend on the variational parameters.
      The second term is
  {\small
  \begin{eqnarray*}
    - \mathbb E_p [\log q({\bf z})] = \frac{1}{2} \log |\tilde{\boldsymbol \Psi}| + \frac{1}{2}\mathbb E_p ({\bf z} - \tilde{\boldsymbol \nu})^T \tilde{\boldsymbol \Psi}^{-1} ({\bf z} - \tilde{\boldsymbol \nu}) \\
    = \frac{1}{2} \sum_{i = 1}^n \log \tilde \Psi_{ii} +
      \frac{1}{\tilde \Psi_{ii}} \left (\Sigma_{ii} + (\mu_i - \tilde \nu_i)^2 \right).
  \end{eqnarray*}
  }
  This expression is minimized by setting $\tilde \nu_i = \mu_i$ and moreover $\tilde{\boldsymbol \nu} = \boldsymbol \mu$.
  Differentiating with respect to $\tilde \Psi_{ii}$ and solving at a stationary point, we then have $\tilde \Psi_{ii} = \Sigma_{ii}$.
  \end{proof}
 
  This next theorem, obtained when minimizing the reverse KL-divergence, is the counterpart to Theorem~\ref{thm:constant-off-diag}.

  \begin{theorem}  \label{thm:constant-off-diag2}
    Suppose ${\bf C}$ has constant off-diagonal terms, $\varepsilon \in [0, 1)$.
    When minimizing the reverse KL-divergence, the entropy gap goes to a constant factor in the limit $n\rightarrow\infty$, whereas the variance is correctly estimated, that is
        \begin{align}
            \underset{n \to \infty} \lim \ \tfrac{1}{n} \left(\mathcal H(p) - \mathcal H(q)\right) &= \log(1 - \varepsilon), \\
            \tilde \Psi_{ii} &= \Sigma_{ii}.
        \end{align}
      \end{theorem}

  \begin{proof}
    The second equality is already stated in Proposition~\ref{prop:solution-reverse}.
    Note that this result does not depend on the specifics of the example at hand and applies to any covariance matrix.
  
    To obtain the first equality, we start with the shrinkage-delinkage decomposition,
    \begin{equation}
        \mathcal H(p) - \mathcal H(q) = \frac{1}{2} \log |{\bf S}| + \frac{1}{2} \log |{\bf C}|.
    \end{equation}
    Since there is no shrinkage, $\log |{\bf S}| = 0$.
    Next, recall from \eqref{eq:logC} the expression for $\log |{\bf C}|$.
    Dividing $\log |{\bf C}|$ by $n$ and taking the limit in $n$, we obtain the desired result.
  \end{proof}
  Naturally, the entropy gap can be arbitrarily large, with $q$ having a larger entropy than $p$, notably as $\varepsilon$ goes to 1.

  For problems where the marginal variances are of interest, rather than the entropy, we would ideally minimize the reverse KL-divergence.
  Unfortunately, there is usually no efficient way to optimize $\text{KL}(p||q)$, due to the difficulty in evaluating expectation values with respect to $p$.

  \section{Solutions for the bounds on the shrinkage and delinkage terms}
  \label{app:alg}

  By exploiting the symmetries that we proved in Section 4.2, we can efficiently compute bounds on the terms $\log |\mathbf{S}|$ and $\log|\mathbf{C}|$; these are the terms that arise, respectively, from the effects of shrinkage and delinkage.

  \subsection{Upper bound on $\log |{\bf S}|$} 
  
  First we show how to compute the upper bound on $\log |\mathbf{S}|$. Recall that to do so, we must solve the optimization problem
  \begin{equation}
      \max_{\boldsymbol\lambda\in\Lambda_R} \sum_{i=1}^n \lambda_i^{-1}.
      \label{eq:optS}
  \end{equation}
 From Lemma 4.2, we know that all the elements of the solution assume the edge values of $\lambda_1$ or $\lambda_n$
  save for at most one which we denote $\lambda_k$. At a high level, we solve the optimization by
 exhaustively computing the optimal solution for each candidate value of $k \in \{1,\ldots, n\}$,
  then choosing the particular value of $k$ whose solution maximizes the overall objective function.

  It remains only to show how to compute the solution for a particular candidate value of $k$.
  Recall the constraints that $\sum_{i = 1}^n \lambda_i = n$ and $\lambda_1 = R \lambda_n$.
 It follows that
  \begin{equation}
      \lambda_k = n - \left [ (k - 1) R + n - k \right] \lambda_n.
  \end{equation}
  Using the constraints to eliminate $\lambda_1$ and $\lambda_k$, we can write the
 objective function entirely in terms of $\lambda_n$. In this way we find
  {\small
  \begin{equation} \label{eq:objective-simple}
      \sum_{i = 1}^n \lambda_i^{-1} = \frac{1}{n - \left [ R (k - 1) + n - k \right] \lambda_n} + \frac{(k - 1)}{R \lambda_n} + \frac{n - k}{\lambda_n}.
  \end{equation}
  }
Crucially, we also need to enforce the boundary conditions $\lambda_n \le \lambda_k \le \lambda_1$, or equivalently
  \begin{equation} \label{eq:boundary}
      \frac{n}{R k + n - k} \le \lambda_n \le \frac{n}{R (k - 1) + n - k + 1}.
  \end{equation}
  Note that the simplified objective in \eqref{eq:objective-simple} for fixed $k$ is convex in $\lambda_n$; hence the maximizer must lie at one of the boundary values in \eqref{eq:boundary}. By computing the objective for each boundary value of $\lambda_k$, we find the optimal solution for this candidate value of~$k$. Finally, we obtain the overall solution to eq.~(\ref{eq:optS}) by considering all $n$ candidate values of $k$ and choosing the best one.

\subsection{Upper Bound on $\log |{\bf C}|$}

Next we show how to compute the upper bound on $\log|\mathbf{C}|$. Recall that to do so,
we must solve the optimization problem
  \begin{equation}
      \max_{\boldsymbol\lambda\in\Lambda_R}\left[\sum_{i=1}^n \log \lambda_i\right].
  \end{equation}
  From Lemma~4.3, we know that all eigenvalues other than~$\lambda_1$ and $\lambda_n$ must have the same value; we denote this value  by~$\lambda_k$.
  From the constraint $\sum_{i = 1}^n \lambda_i = n$, it follows that
  \begin{equation}
      \lambda_k = \frac{n - (1 + R) \lambda_n}{n - 2}.
  \end{equation}
  Again, using the constraints to eliminate $\lambda_1$ and $\lambda_k$, we can write the objective function entirely in terms of $\lambda_n$. In this way we find
  \begin{equation}
      \sum_{i=1}^n \log \lambda_i = (n\! -\! 2) \log \frac{n - (1 + R) \lambda_n}{n - 2} + \log R \lambda_n + \log \lambda_n.
      \label{eq:concave}
  \end{equation}
This objective is concave in $\lambda_n$, so we can locate the maximum by setting its derivative with respect to $\lambda_n$ equal to zero. Some straightforward algebra shows that this derivative vanishes when
  \begin{equation}
      \lambda_n = \frac{2}{1+R}.
  \end{equation}
Finally we need to check that this solution does not violate the boundary conditions of the problem; in particular, we require that \mbox{$\lambda_n \ge \lambda_k \ge R \lambda_n$}, or equivalently that
  \begin{equation}
      \frac{n}{1 + R (n - 1)} \le \lambda_n \le \frac{n}{n - 1 + R}.
  \end{equation}
These conditions are always satisfied for $n\geq 3$.
  Hence we obtain an analytical solution for the upper bound on $\log|\mathbf{C}|$.
  Finally, note that while the solution for $\lambda_n$ does not depend on $n$,
  the optimized objective function does depend on $n$ through eq.~(\ref{eq:concave}).

 Algorithm~\ref{alg:bounds} provides an implementation of the above-described method.

\begin{algorithm}[!b]
    \DontPrintSemicolon
    \caption{Upper bounds on $\log|\mathbf{S}|$ and $\log|\mathbf{C}|$}
    \label{alg:bounds}
    \setstretch{1}
    {\bf Input:} $R, n$ \;\;

    \SetKwFunction{Fh}{ObjF}
    \SetKwProg{Fn}{Function}{:}{}
    \Fn{\Fh{$\lambda_n$, $k$}}{
        \KwRet $ \left (n - k + \frac{k - 1}{R} \right)\frac{1}{\lambda_n} + \frac{1}{n - [R(k - 1) + n - k] \lambda_n}$
    } \;

    \For {$k$ in $\{2, \cdots, n - 1\}$} {
      $\lambda_a \leftarrow \frac{n}{Rk + n - k}$ \;
      $\lambda_b \leftarrow \frac{n}{R(k-1) + n-k + 1}$ \;
      $F_k \leftarrow \text{max}(\text{\texttt{ObjF}} \;(\lambda_a), \text{\texttt{ObjF}}(\lambda_b))$ \;
      {\bf if} ($k = 1$) $F \leftarrow F_k$ \;
      {\bf else} $F \leftarrow \text{max}(F, F_k)$\;
    }
    $U_s \leftarrow n\log(F/n)$\;\;
    
    $\lambda_n \leftarrow \frac{2}{1+R}$ \;
    $U_c \leftarrow  \log \frac{1}{\lambda_n} + \log\frac{1}{R \lambda_n} + (n\!-\!2) \log \frac{n-2}{n-(1+R)\lambda_n}$\; \;

    {\bf Return:} $U_s$, $U_c$\;
  \end{algorithm}
  

\section{Bounds on the average variance shrinkage}
\label{app:traceS}

We can also derive bounds on the {\it average} shrinkage in componentwise variance in terms of the problem dimensionality, $n$, and the condition number, $R$, of the correlation matrix. The average in this case is performed over the different components of $\mathbf{z}$. Recall that the shrinkage in each componentwise variance is given by $S_{ii} = \Sigma_{ii}/\Psi_{ii}$. Hence we can also express this bound in terms of the trace of the shrinkage matrix, $\text{trace}({\bf S})$.
\begin{proposition}
Suppose that the correlation matrix $\mathbf{C}$ has condition number $R$. Then the solution for FG-VI in section 2 satisfies
\begin{equation}
 \min_{\lambda\in\Lambda_R}\sum_{i=1}^n \lambda_i^{-1} \leq {\rm trace}(\mathbf{S}) \leq \max_{\lambda\in\Lambda_R}\sum_{i=1}^n \lambda_i^{-1},
\label{eq:shrink-eig-bound}
\end{equation}
where $\Lambda_R$ is the set defined in section 4.
\end{proposition}
\begin{proof}
We showed in the proof of Theorem 3.1 that 
\begin{equation}
    S_{ii} = \frac{\Sigma_{ii}}{\Psi_{ii}}= C_{ii}^{-1}.
\end{equation}
It follows that ${\rm trace}(\mathbf{S}) = {\rm trace}(\mathbf{C}^{-1}) = \sum_i \lambda_i^{-1}$, where $\lambda_1,\ldots,\lambda_n$ are the eigenvalues of $\mathbf{C}$. The bound then follows from the relaxtion from the set $\mathcal{C}_R$ to the set $\Lambda_R$ in section~4.
\end{proof}

\subsection{Lower bound on $\text{trace}({\bf S})$}

The optimization implied by eq.~(\ref{eq:shrink-eig-bound}) is convex, since both the set $\Lambda_R$ and the objective function $\sum_i \lambda_i^{-1}$ are convex. In fact, this bound can be evaluated in closed form by using similar methods as in section 4.2.

  \begin{lemma} \label{lemma:trace-symmetry1}
  Let $\boldsymbol\lambda\in\Lambda_R$ be the solution that minimizes the left side of eq.~(\ref{eq:shrink-eig-bound}). Then $\lambda_i\!=\!\lambda_j$ whenever $1\!<\!i\!<\!j\!<\!n$.
\end{lemma}
\begin{proof}
This proof follows the same argument as the proof (by contradiction) for Lemma 4.3. 
Suppose there exists a solution with intermediate elements $\lambda_i$ and $\lambda_j$ that satisfy
$\lambda_1\! \geq\! \lambda_i\! >\! \lambda_j\! \geq\! \lambda_n$.
Consider the effect on this solution of a perturbation that adds some small amount $\delta\!>\!0$ to~$\lambda_j$ and subtracts the same amount from~$\lambda_i$. For sufficiently small~$\delta$, this perturbation will not leave the set~$\Lambda_R$; however, it will {\it diminish} the separation of~$\lambda_i$ from~$\lambda_j$. As a result the objective $\sum_k (1/\lambda_k)$ experiences a change
\begin{equation}
  g(\delta) = \frac{1}{\lambda_i-\delta} -\frac{1}{\lambda_i} +\frac{1}{\lambda_j+\delta} - \frac{1}{\lambda_j}.
  \end{equation}
Evaluating the derivative, we find $g'(0) = \lambda_i^{-2}\! -\! \lambda_j^{-2} < 0$, so that the objective function is decreased for some $\delta>0$. As before this yields a contradiction, because any solution must be maximal, and hence stationary (i.e., $g'(0)\!=\!0$), with respect to small perturbations.
\end{proof}
With the above lemma, the $n$-dimensional optimization over $\Lambda_R$ can be reduced to a one-dimensional optimization that can be solved in closed form. The methods are identical to those in the previous appendix.

First we rewrite the constraint, $\lambda_n \le \lambda_k \le R \lambda_n$, as
\begin{equation}  \label{eq:trace-constraint}
    \frac{n}{R(n - 1) + 1} \le \lambda_n \le \frac{n}{R  + n - 1}.
\end{equation}
Since the minimization problem is convex, a minima can be found at a stationary point of the objective function
\begin{equation}  \label{eq:trace-objective}
    \sum_{i = 1}^n \frac{1}{\lambda_i} = \frac{(n - 2)^2}{n - (1 + R) \lambda_n} + \frac{1}{\lambda_n} + \frac{1}{R \lambda_n},
\end{equation}
which now only depends on $\lambda_n$.
Differentiating and setting to 0, we obtain the root-finding problem,
\begin{equation}
    \left [R(n - 2)^2 - (1 + R)^2 \right] \lambda^2_n + 2n (1 + R) \lambda_n - n^2 = 0,
\end{equation}
which can be solved exactly.
It remains to check whether the roots violate the constraints in~\eqref{eq:trace-constraint}, and pick the non-offending root which maximizes the objective in~\eqref{eq:trace-objective}.
If both roots violate the constraints then, by convexity of the problem, the solution must lie at one of the boundary terms in~\eqref{eq:trace-constraint}.

\subsection{Upper bound on \text{trace}(S)}

A similar approach gives us an upper bound on $\text{trace}({\bf S})$.
In fact, we have already solved the problem of maximizing the right side of~\eqref{eq:shrink-eig-bound} when upper-bounding $\log |{\bf S}|$.
It remains to apply the same strategy.

\section{Tighter upper bound on entropy gap}
\label{app:KL}
In Proposition 4.1 we derived separate upper bounds on the individual terms $\log|\mathbf{S}|$ and $\log|\mathbf{C}|$. One upper bound on the entropy gap (or equivalently, on ${\rm KL}(q,p)$) is obtained simply by adding these separate bounds. However, a tighter upper bound is obtained by replacing the separate optimizations in Proposition 4.1 by a single joint optimization:
\begin{equation}
{\rm KL}(q,p)\ \leq\ 
  \frac{1}{2}\max_{\boldsymbol\lambda\in\Lambda_R}\left[
    n \log \frac{1}{n}
  \sum_{i=1}^n \lambda_i^{-1} + \sum_{i=1}^n \log \lambda_i\right].
\label{eq:combined-opt}
\end{equation}
In this appendix we sketch how to solve this optimization and evaluate this bound in closed form. The first step is to make the change of variables,
\begin{equation}
    \omega_i = \frac{\lambda_i^{-1}}{\sum_{j=1}^n \lambda_j^{-1}},
\end{equation}
and to translate the domain of optimization accordingly. Under this change of variables, the original domain $\Lambda_R$ in section 4 is mapped onto the set
\begin{equation}
    \Omega_R = \left\{\boldsymbol\omega\in \mathbb R_+^n\, |\,\omega_n\geq\ldots\geq\omega_1 = \frac{1}{R} \omega_n, \sum_{i=1}^n\omega_i = 1\right\}.
\end{equation}
Likewise, a little algebra shows that the optimization in eq.~(\ref{eq:combined-opt}) is equivalent to the following:
\begin{equation}
{\rm KL}(q,p)\ \leq\ 
  \frac{1}{2}\max_{\boldsymbol\omega\in\Omega_R}\left[
    \sum_{i=1}^n \log\frac{1}{\omega_i} - n\log n\right].
\label{eq:omega-opt}
\end{equation}
Now we can make a similar argument as in the proof of Lemma 4.2 to simplify this optimization.
\begin{lemma} 
\label{lemma:symmetry-omega}
  Let $\boldsymbol\omega\in\Omega_R$ be the solution that maximizes the right side of eq.~(\ref{eq:omega-opt}). Then at most one $\omega_i$ is not equal to either $\omega_1$ or~$\omega_n$.
\end{lemma}
\begin{proof}
We prove the lemma by contradiction. Suppose there exists a solution with intermediate elements $\omega_i$ and $\omega_j$ that satisfy
$\omega_n\! >\! \omega_i\! >\! \omega_j\! >\! \omega_1$.
Consider the effect on this solution of a perturbation that adds some small amount $\delta\!>\!0$ to $\omega_i$ and subtracts the same amount from $\omega_j$. Note that for sufficiently small~$\delta$, this perturbation will not leave the set~$\Omega_R$; however, it will {\it expand} the separation of $\omega_i$ from~$\omega_j$. As a result the objective in eq.~(\ref{eq:omega-opt}) changes by an amount
\begin{equation}
  f(\delta) = \frac{1}{2}\left[\log\frac{1}{\omega_i\!+\!\delta} - \log\frac{1}{\omega_i} + \log\frac{1}{\omega_j\!-\!\delta} - \log\frac{1}{\omega_j}\right].
\end{equation}
Next we evaluate the derivative $f'(\delta)$ at $\delta=0$; doing so we find $f'(0) = \omega_j^{-1}\! -\! \omega_i^{-1} > 0$, so that the objective is increased for some $\delta>0$. But this yields a contradiction, because any solution must be maximal, and hence stationary (i.e., $f'(0)\!=\!0$), with respect to small perturbations.
\end{proof}
With the above lemma, we can reduce the $n$-dimensional optimization over $\Omega_R$ to a one-dimensional optimization that can be solved in closed form; the methods are identical to those in the previous appendix.

In details, let $\omega_k$ be the one variable which (potentially) does not go to $\omega_1$ or $\omega_n$.
Given $\sum_{i = 1}^n \omega_i = 1$,
\begin{equation}
    \omega_k = 1 - (k - 1 + R(n -k)) \omega_1.
\end{equation}
The objective is then
\begin{eqnarray}
    \sum_{i = 1}^n \log \frac{1}{\omega_i} =&  - (k - 1) \log \omega_1 - (n - k) \log R\omega_1 \nonumber \\ & - \log (1 - [k - 1 + R(n -k) \omega_1]). \ 
\end{eqnarray}
Since we are trying to maximize a convex function, the solution does not lie at a stationary point, rather at a boundary set by the constraint, $\omega_1 \le \omega_k \le \omega_n$, or equivalently
\begin{equation}
    \frac{1}{k - 1 + R(n - k + 1} \le \omega_1 \le \frac{1}{k + R(n -k)}.
\end{equation}
It remains to test each candidate boundary for each choice of $k$ to obtain a maximizer.

\end{document}